\documentclass[11pt]{article}

\usepackage[preprint]{acl}

\usepackage{times}
\usepackage{latexsym}

\usepackage[T1]{fontenc}

\usepackage[utf8]{inputenc}

\usepackage{microtype}

\usepackage{inconsolata}

\usepackage{booktabs}
\usepackage{multirow}    
\usepackage[table]{xcolor} 
\usepackage{array}       
\usepackage{graphicx}    
\usepackage{makecell}    
\usepackage{graphicx}
\usepackage{amsfonts}
\usepackage{amsmath}
\usepackage{bbm}
\usepackage{dblfloatfix}
\usepackage{float}
\usepackage{amsthm}
\usepackage{thmtools, thm-restate}
\usepackage{wrapfig}
\usepackage{placeins}
\usepackage{algorithm}
\usepackage{algorithmicx}
\usepackage{algpseudocode}
\usepackage{dblfloatfix}
\newtheorem{theorem}{Theorem}[section]
\newtheorem{lemma}[theorem]{Lemma}

\theoremstyle{definition}

\theoremstyle{remark}

\newcommand{\arrol}{\textsc{ARRoL}}

\title{Prune as You Generate: Online Rollout Pruning for\\ Faster and Better RLVR}

\author{\textbf{Haobo Xu$^1$, Sirui Chen$^1$, Ruizhong Qiu$^1$, Yuchen Yan$^2$,}\\ {\textbf{Chen Luo$^2$, Monica Cheng$^2$, Jingrui He$^1$, Hanghang Tong$^1$}} \\
$^1$ University of Illinois at Urbana-Champaign\ 
   $^2$ Amazon\\
  }

\begin{document}
\maketitle

\begin{abstract}
Reinforcement Learning with Verifiable Rewards (RLVR) has significantly advanced the reasoning capabilities of Large Language Models (LLMs). However, methods such as GRPO and DAPO suffer from substantial computational cost, since they rely on sampling many rollouts for each prompt.
Moreover, in RLVR the relative advantage is often sparse: many samples become nearly all-correct or all-incorrect, yielding low within-group reward variance and thus weak learning signals.
In this paper, we introduce \arrol\ (\textbf{A}ccelerating \textbf{R}LV\textbf{R} via \textbf{o}nline Ro\textbf{L}lout Pruning), an online rollout pruning method that prunes rollouts during generation while explicitly steering the surviving ones more correctness-balanced to enhance learning signals. Specifically, \arrol\ trains a lightweight quality head on-the-fly to predict the success probability of partial rollouts and uses it to make early pruning decisions. The learned quality head can further weigh candidates to improve inference accuracy during test-time scaling. 
To improve efficiency, we present a system design that prunes rollouts inside the inference engine and re-batches the remaining ones for log-probability computation and policy updates. Across GRPO and DAPO on Qwen-3 and LLaMA-3.2 models (1B-8B), \arrol\ improves average accuracy by $+2.30$ to $+2.99$ while achieving up to $1.7\times$ training speedup, and yielding up to $+8.33$ additional gains in average accuracy in test-time scaling. The code is available at \url{https://github.com/Hsu1023/ARRoL}.
\end{abstract}
\section{Introduction}

\begin{figure*}
\begin{center}
\includegraphics[scale=1.25]{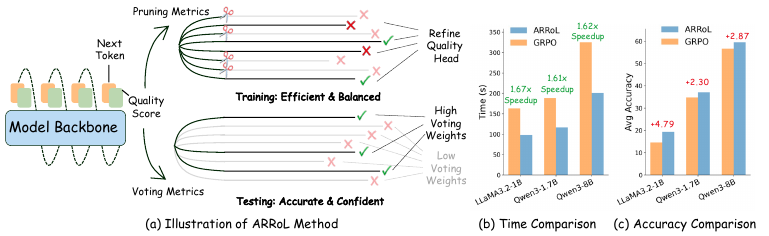}
\caption{\textbf{\arrol\ overview and results.} (a) \arrol\ uses a quality head to score partial rollouts, enabling early pruning for efficient and reward-balanced training, and the scores can also be used as voting weights for test-time scaling. (b) Wall-clock time comparison between \arrol\ and GRPO across different model backbones, showing consistent speedups. (c) Accuracy comparison, where \arrol\ improves average accuracy over GRPO.}
\vspace{-10pt}
\end{center}
\end{figure*}
Reasoning abilities of Large Language Models (LLMs) have recently gained great success in many domains such mathematical problem solving and code generation~\citep{jaech2024openai,guo2025deepseek,qiu2024efficient}. 
Reinforcement Learning with Verifiable Rewards (RLVR)~\citep{lambert2024tulu,dapo} as a critical technique plays an important role in enhancing reasoning ability of LLMs. A representative method is Group Relative Policy Optimization (GRPO)~\citep{grpo}, which utilizes binary rewards such as correctness of a logical problem as learning signals and compute advantages within a group of rollouts per prompt. However, such methods are largely constrained by high computational cost~\citep{xu2025not,lin2025cppo}. During training, each prompt requires generating a large group of rollouts, which is computationally expensive, making training expensive and limiting the practicality of RLVR at scale. 

To mitigate training cost, prior work has explored several directions. Some studies~\citep{lin2025cppo,xu2025not} 
reduce the number of rollouts used for gradient estimation and policy updates. 
However, these methods typically manipulate rollouts at the post-generation level; therefore, they do not reduce rollout generation time, which can limit end-to-end speedups.

Other studies employ speculative decoding to accelerate rollout generation~\citep{he2025history,liu2025spec}, but they rely on historical sequences from previous epochs, which may vary and not suitable for commonly adopted small epochs settings. 
Moreover, they do not explicitly address a key issue in RLVR with binary rewards (0/1): when rewards within a group are highly imbalanced (e.g., mostly correct or mostly incorrect), the within-group reward diversity becomes low, leading to weak learning signals~\citep{bae2025online}. In the extreme case where a group collapses to all 0s or all 1s, the group-normalized advantages can degenerate to zero, resulting in a vanishing policy gradient. This motivates a key question: \emph{Can we reduce rollout cost while strengthening learning signals?}

To answer this question, we propose an online rollout pruning method that \emph{carefully selects a correctness-balanced subset of rollouts for training during rollout generation} using a quality predictor. 
Concretely, we train a lightweight model head 
to score early-stage partial rollouts and map the score to an estimated success probability. The rollouts with final rewards can naturally be used as training data of rollouts, and 
it introduces negligible overhead. 

We compare the quality scores with other heuristic metrics, such as DeepConf~\citep{fu2025deep}, and show that the scores generated by a learnable head are better than trace confidence across datasets, as the latter can be biased by patterns (e.g., reflections or formula-rich text) and may misalign with final correctness. 
Then, we prune the early-stage rollouts based on the quality scores to make the correctness of remaining rollouts more balanced. This yields a “less is more” effect:
fewer rollouts continue to generate and involve in advantage computation, leading to less computation cost, while the remaining rollouts provide stronger learning signals due to improved balance. Furthermore, the learned quality head can naturally serve as a correctness predictor to weight candidates in \textit{test-time scaling} to improve accuracy instead of naive majority vote. 

To implement online pruning during generation and improve efficiency, we integrate pruning into a standard frontend–backend RL training architecture. The backend evaluates rollouts by the quality head at an early stage and immediately removes pruned sequences from the request pool, allowing the scheduler to reallocate freed capacity to other active sequences and reduce overall generation time. The frontend receives pruning masks, filters pruned rollouts, and re-batches the survivors for log-probability computation and optimization, leading to less computational cost as well.

In summary, our key contributions are as follows: 
\begin{itemize}
    \item \textbf{Online Rollout Pruning.} We propose an online, quality-head-guided rollout pruning strategy, \arrol, which explicitly controls within-group reward balance while reducing compute, improving average accuracy of GRPO/DAPO training on Qwen-3 and LLaMA-3.2 models (1B-8B) by $+2.30$ to $+2.99$.
    \item \textbf{Test-time Scaling.} We leverage the trained quality head at inference time as voting weights for test-time scaling, improving final-answer aggregation, leading to $+8.33$ gains in average accuracy.
    \item \textbf{System Speedup.} We present a system design that realizes end-to-end speedups by pruning inside the generation backend and re-batching survivors in the frontend, achieving $1.6-1.7\times$ speedup. 

\end{itemize}

\section{Related Work}
\paragraph{Efficient RLVR.} 
Recent studies improve the efficiency of RLVR from different angles.
Some modify rollout construction. For example, S-GRPO~\citep{dai2025s} derives a serial group of rollouts from a single trajectory, which may reduce trajectory diversity.
Others leverage historical information. GRESO~\citep{zheng2025act} skips likely uninformative prompts before rollouts.
RhymeRL~\citep{he2025history} reuses historical rollout tokens by speculative decoding. However, these speedups rely on historical information and can be limited in cold-start settings.
Another line targets post-rollout optimization.
CPPO~\citep{lin2025cppo} prunes generated completions, and PODS~\citep{xu2025not} downsamples a large rollout pool, to accelerate the update phase. However, these methods do not directly reduce (and can even increase) token-generation cost during rollout.
Several works target to accelerating rollout generation.
Spec-RL~\citep{liu2025spec} and FastGRPO~\citep{zhang2025fastgrpo} employ speculative decoding, while FlashRL~\citep{yaoyour} and QeRL~\citep{huang2025qerl} use low-precision/quantized rollouts to speed up token generation.
In contrast, our method uses a lightweight logits-based probe to predict rollout utility online, enabling training-time pruning with controlled signal quality and a unified criterion that also supports test-time rollout filtering.

\paragraph{Test-time Scaling.} 
Test-time scaling (TTS) improves reasoning performance by allocating additional compute
at inference time without modifying model parameters.
TTS is commonly categorized into sequential and parallel strategies.
Sequential TTS increases compute along a single trajectory by extending reasoning process or revisiting the initial answer.
For instance, s1~\citep{muennighoff2025s1} proposes budget forcing to terminate trajectories early or append double-check tokens to encourage rethinking.
Other work studies the underthinking phenomenon and triggers deeper deliberation if necessary~\citep{wang2025thoughts,qiu2025ask}.
In contrast, parallel TTS samples multiple trajectory candidates and aggregates them, including Self-Consistency~\citep{wang2022self,cui2026adafuse}, Best-of-N~\citep{zhou2022least,qiu2025efficient2}, and adaptive voting~\citep{snell2024scaling}.
Recent studies further explore confidence-based TTS.
DeepConf~\citep{fu2025deep} uses log-probability-based confidence to prune low-confidence trajectories, while CGES~\citep{aghazadeh2025cges} employs heuristic estimates or reward models to early-stop the sampling process.
However, these heuristic confidence signals are typically not guaranteed to align with final-answer correctness, so their reliability may degrade under distribution shift or in out-of-domain settings.

\section{Preliminaries}
\paragraph{Trace Confidence.}
Recent work leverages model-internal uncertainty signals to evaluate the quality of an LLM-generated trace.
Given the predicted token distribution $P_t(\cdot)$ at position $t$, token confidence is defined as
$H_t = -\sum_{j=1}^{V}\log P_t(j)$, where $V$ is the vocabulary size. Self-uncertainty~\citep{kang2025scalable} defines trace confidence as the average token confidence over the trace. 
DeepConf~\citep{fu2025deep} further improves effectiveness and efficiency by computing window-level confidence. 
Specifically, it averages token uncertainty within a fixed-size sliding window $w$:
$H_w = \frac{1}{|w|}\sum_{t\in w} H_t$,
and then aggregates $\{H_w\}$ along the trace, e.g., using the minimum window value or the average of the bottom 10\% windows as the trace-level score.

\paragraph{Reinforcement Learning with Verifiable Rewards (RLVR).} 
Let the large language model be the policy $\pi_\theta$ that, given a prompt $x$, generates a rollout $o$ that contains the reasoning trace and the final answer $y$.
Assume a dataset $\mathcal{D}=\{(x_i,a_i)\}_{i=1}^N$, where $a_i$ is the ground-truth answer to the prompt $x_i$.
We define a verifiable reward $R(y_i,a_i)=\mathbbm{1}[y_i \equiv a_i]$, where $\mathbbm{1}[\cdot]\in\{0,1\}$ is an indicator that evaluates whether $y_i$ and $a_i$ are equivalent, e.g., mathematically equivalent.

\paragraph{Group-Relative Policy Optimization (GRPO).} GRPO~\citep{grpo} estimates advantages using the relative performance within a group of answers for the same prompt, without a value model. The objective is:

\vspace{-30pt}
\begin{equation*}
\begin{aligned}
J(\theta)& = \mathbb{E}_{(x,a)\sim \mathcal{D},\{o_i\}_{i=1}^{G}}\frac{1}{G}\sum^{G}_{i=1}\frac{1}{|o_i|}\sum^{|o_i|}_{t=1}\min[r_{i,t}A_{i},\\
&
\text{clip}(r_{i,t}, 1-\epsilon, 1+\epsilon)A_{i}]-\beta \cdot\text{KL}(\pi_\theta||\pi_{ref}),
\end{aligned}
\end{equation*}
where $G$ is the group size,
$r_{i,t}=\frac{\pi_\theta(o_{i,t}\mid o_{i,<t})}{\pi_{\text{ref}}(o_{i,t}\mid o_{i,<t})}$ is the importance ratio, and
$A_{i}=\frac{R(y_i,a_i)-\mathrm{mean}(\{R(y_j,a_j)\}_{j=1}^{G})}{\mathrm{std}(\{R(y_j,a_j)\}_{j=1}^{G})}$ is the group-relative advantage. However, GRPO incurs substantial time cost because it must generate many rollouts per prompt and process them for log-probability computation and policy updates. Also, it can encounter sparse signal issue that rewards within a group are all 0/1 and group-normalized advantages become zero, leading to vanishing gradient. 
\section{Method}

We introduce \arrol, a rollout pruning method to improve the efficiency of GRPO by balancing 0/1 rewards within each group. \arrol\ trains a lightweight model head, named quality head,  to score partial rollouts, and uses these scores to select a balanced subset for training. The learned scores can also be reused as voting weights at test time. We further present a system design that realizes the wall-clock speedup in practice.

\subsection{Pruning Improves Sample Balance}
\label{sec:balanced}

GRPO suffers from sparse signals when the rewards are nearly all 0s or all 1s. Intuitively, a more balanced 
sample group will introduce larger variance within a group, leading to non-vanishing gradients. Also, if samples are balanced, we can avoid circumstances in which some groups are dominated by a few minority samples, leading to sparse and noisy learning signals. Recent studies~\citep{bae2025online} provide theoretical support that, under binary (0/1) rewards, the learning signal is proportional to the reward variance and is maximized when the pass ratio (i.e., the fraction of positive samples within a group) is close to 0.5, thereby improving the effectiveness of RL training. 
Based on this, letting the positive-sample ratio be $\rho$, we further show that rollout pruning can push the empirical ratio toward $\rho$ (ideally $\rho=0.5$), improving sample balance beyond its efficiency gains.

\begin{lemma}[Existence of a Corrective Pruning]\label{lemma:1}
Consider a mini-batch of size $G$, each with a label $y_i\in\{0,1\}$. We assume a fixed positive ratio $\rho\in [0,1]$ and conditional independence given latent Bernoulli parameters: $Y_i|q^\star_i\sim \text{Bernoulli} (q^\star_i)$.  We define the batch mean
$\mu^\star := \frac{1}{G}\sum_{i=1}^G q_i^\star$ and the pruned mean
$\mu_{-j}^\star := \frac{1}{G-1}\sum_{i\neq j} q_i^\star$.
If $\mu^\star>\rho$ and there exists an index $j$ such that $q_j^\star>\mu^\star$,
then pruning $j$ strictly reduces the deviation to $\rho$:
\[
\big|\mu^\star_{-j}-\rho\big| < \big|\mu^\star-\rho\big|.
\]
Symmetrically, if $\mu^\star<\rho$ and there exists $j$ such that $q_j^\star<\mu^\star$,
then the same conclusion holds.
\end{lemma}

\begin{theorem}[High-probability closeness to target $\rho$]\label{theorem:1}
Under the setting of Lemma~\ref{lemma:1}.
Assume we have posterior-mean estimates $\{q_i\}_{i=1}^G$
satisfying the uniform accuracy condition $|q_i-q_i^\star|\le \epsilon$, $\forall i$.
We define
\(
\mu_{-j}:=\frac{1}{G-1}\sum_{i\ne j} q_i,
\)
let
\(
\hat{\jmath}:=\arg\min_{j\in[G]}\big|\mu_{-j}-\rho\big|\), 
\(\hat{p}_{-\hat{\jmath}}:=\frac{1}{G-1}\sum_{i\ne \hat{\jmath}} Y_i,
\)
and fix any $\delta\in(0,1)$. Then, with probability at least $1-\delta$, we have
\[
\big|\hat{p}_{-\hat{\jmath}}-\rho\big|
\ \le\
\min_{j\in[G]}\big|\mu^\star_{-j}-\rho\big|
\ +\ 2\epsilon
\ +\ \sqrt{\frac{\log(2/\delta)}{2(G-1)}}.
\]
\end{theorem}

\noindent
The proof can be found in Appendix~\ref{appx:proof}. It implies that pruning can reduce the posterior-mean deviation to a target ratio $\rho$, and if the posterior estimates $q_i$ are accurate, then posterior-guided pruning is $O(\epsilon)$-close to the true-posterior pruning in terms of deviation from $\rho$. Therefore, setting $\rho=0.5$ can enhance the learning signals and improve the effectiveness of training. 
\subsection{Quality Prediction Head}
\label{sec:quality}
\begin{figure*}[t]
\begin{center}
\includegraphics[scale=0.55]{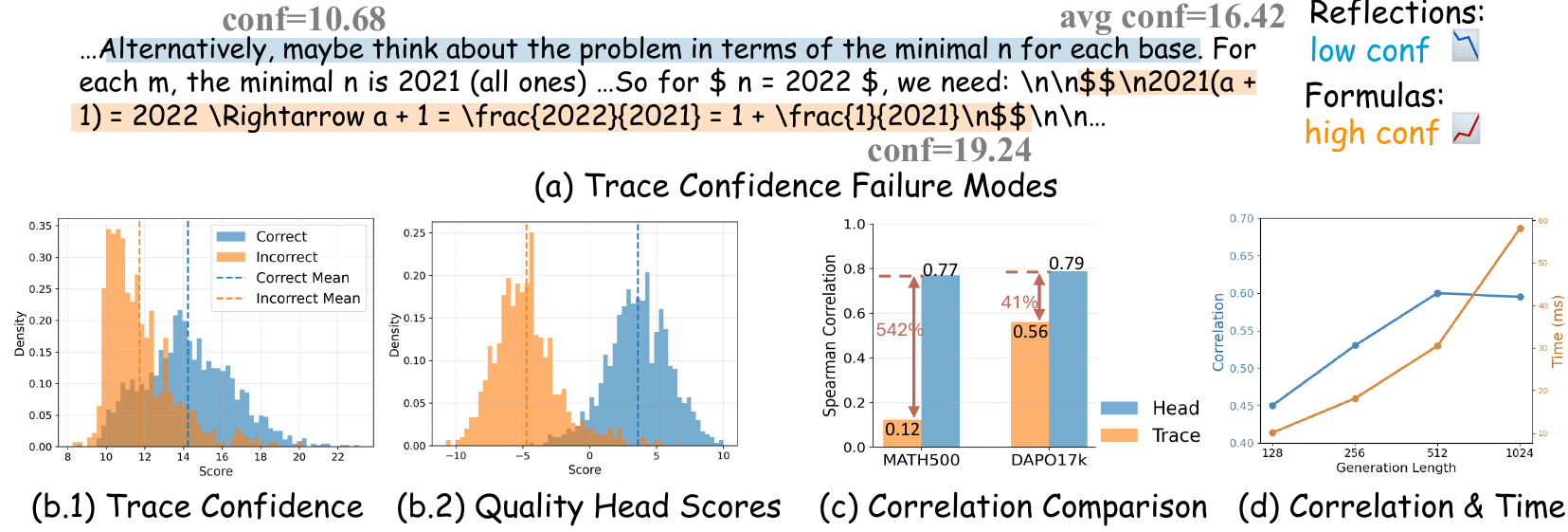}
\end{center}
\vspace{-5pt}
\caption{(a) Trace Confidence Failure Modes: Reflection-related tokens tend to receive low confidence despite being beneficial, whereas formula-heavy tokens can receive high confidence even under incorrect reasoning. (b) Distribution Comparison. Trace confidence in (b.1) is less separable between correct/incorrect than quality head scores in (b.2).
(c) Correlation Comparison. Quality-head scores achieve consistently higher correlation, measured by Spearman rank correlation between the predicted scores (quality scores or trace confidence) and the binary correctness of final answers on the Math500 and Dapo17k datasets. (d) Generation Length v.s. Correlation \& Time Cost. The time cost increases as the generation length increases, while the correlation plateaus when the length reaches 512. All the data is generated by Qwen3-4B model on 400 prompts from Dapo17k and Math500 dataset with 10 rollouts per sample.
\vspace{-8pt}
} 
\label{fig:conf}
\end{figure*}

We have shown that if we have an accurate posterior estimate $q_i$ for each rollout, we can prune rollouts to control the within-group positive ratio and improve the learning signal.
However, $q_i$ is not directly observable during generation, especially when we want to prune a rollout early.
To operationalize Sec.~\ref{sec:balanced}, we first construct an early-stage \textit{quality score $s_i$} for a partial rollout, and then map this score to a posterior estimate $q_i\in[0,1]$.

\paragraph{Quality Score Prediction.} 
Previous studies~\citep{kang2025scalable,fu2025deep} introduce internal uncertainty signals, trace confidence, based on the log-probability of next tokens, which can serve as \textit{quality scores $s_i$}. 
However, these metrics are only indirect proxies of rollout quality. Since they are computed from token-level likelihood without task supervision, they are not guaranteed to align with the final success label.
As shown in Fig.~\ref{fig:conf}(a), a failure example indicates that reflection-related tokens tend to receive low confidence despite being beneficial, whereas formula-heavy tokens can receive high confidence even under incorrect reasoning. 
Therefore, we turn to the model’s hidden representations to generate quality scores of rollouts.
Like the next token prediction head in language models, we can also add a \textit{rollout quality head} to the backbone model, which can be a simple 2-layer MLP. Since RL naturally provides labeled rollouts, we can train the quality head on-the-fly using cross-entropy loss, whose gradient will be detached from backbone model to avoid possible overhead.
To evaluate the effectiveness of the scores from quality head, we collect 4,000 rollouts from two datasets.
As shown in Fig.~\ref{fig:conf}(b), the scores given by quality head as quality scores can distinguish the correct rollouts from incorrect ones, with separable distributions of the two categories. Also, quality head scores show a stable Spearman rank correlation across datasets (Fig.~\ref{fig:conf}(c)). 

\paragraph{Detection Length.} Training-time pruning requires choosing an intermediate length to evaluate the quality score. Fig.~\ref{fig:conf}(d) reports the correlation between intermediate quality head scores and final correctness, as well as the generation time to reach each length. We find that early detection is reliable, and choose $L_{\text{detect}}=512$ to balance pruning reliability and time cost.
\paragraph{Probability Calibration.} Given the quality-head score $s_i$, we need a probability-like posterior estimate $q_i\in[0,1]$ to instantiate Sec.~\ref{sec:balanced}.
Since the raw score scale may shift during training, we adopt an \emph{online binned probability estimator} to map scores to posteriors~\citep{zadrozny2001obtaining}. 
Concretely, we first normalize the score to $s'_i\in[0,1]$ and assign it to one of $B$ uniform bins denoted as $b(s')$.
For each bin, we maintain the numbers of historical positive and negative rollouts (with a sliding buffer), and estimate the posterior success probability by:
\begin{equation*}
\begin{split}
q(&s')=P(Y=1|b(s'))\\&=\frac{\pi P(b(s')|Y=1)}{\pi P(b(s')| Y=1)+(1-\pi)P(b(s')|Y=0)},
\end{split}
\end{equation*}
where $\pi=P(Y=1)$, $P(b(s')| Y=0)$ and $P(b(s')| Y=1)$ are estimated by historical information from previous steps maintaining with a sliding buffer. 

With $q_i=P(y=1|s_i)$ as an estimate of the rollout success probability, we assign each rollout a \emph{survival probability} $p_i$.
The design goal is two-fold: (i) the expected keep ratio matches a target $\kappa$, and (ii) the kept rollouts have a controlled positive ratio close to $\rho$.
We achieve this by defining $p_i$ as a monotonic function of $(\rho - q_i)$ and normalizing it to satisfy the keep-rate constraint.
Sampling rollouts according to $\{p_i\}$ allows us to prune multiple rollouts in one step while steering the within-group balance toward $\rho$. More details can be found in Appendix~\ref{appx:estimator}.

\subsection{System Design}

To enable efficient RL training, we adopt the commonly used framework verl~\citep{sheng2024verl} of frontend-backend architecture and implement rollout pruning inside the generation backend, vLLM~\citep{kwon2023efficient}. An overview is given in Fig.~\ref{fig:system}.
In each training step, we (i) generate rollouts, (ii) compute log-probabilities/advantages, and (iii) update the policy.
The frontend orchestrates data and runs log-probability computation and policy updates, while the backend provides high-throughput rollout generation.
\textbf{(i) Backend.}
The frontend sends rollout-generation requests to the backend.
The backend maintains a request pool and dynamically batches active sequences for GPU execution.
When a rollout first reaches the detection length $L_{\mathrm{detect}}$, the backend evaluates its quality and samples a pruning decision according to the survival probability (Sec.~\ref{sec:quality}).
\begin{figure}
\centering{
\includegraphics[scale=0.47]{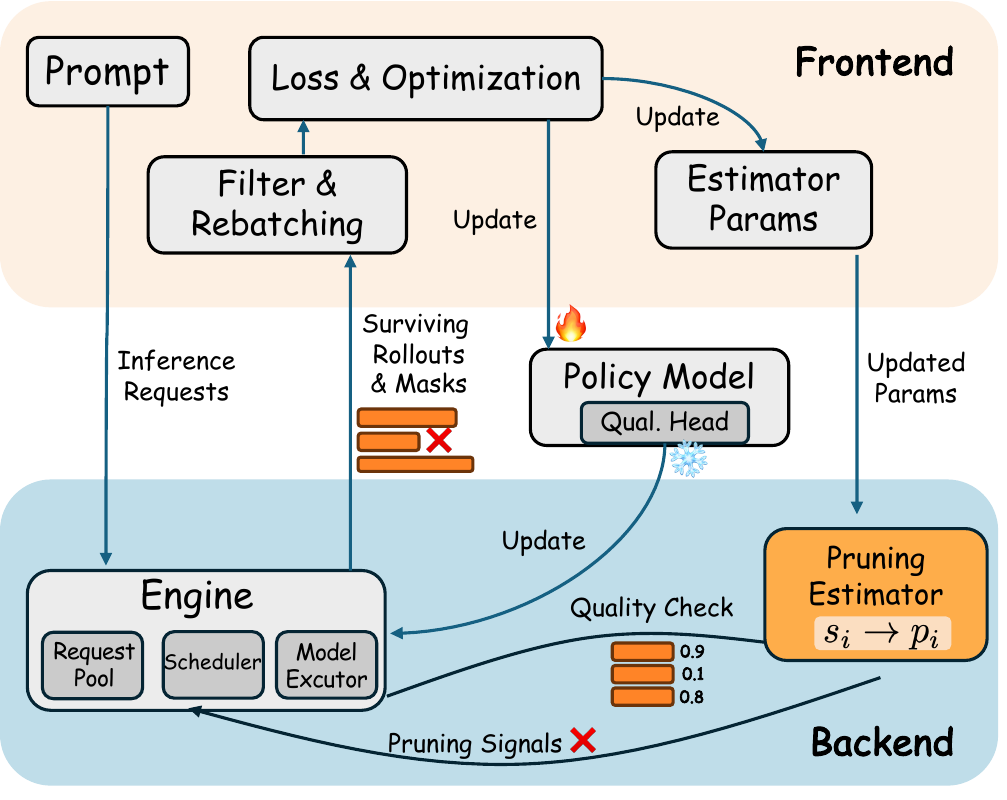}
\caption{Illustration of System Design.}
\label{fig:system}}
\vspace{-10pt}
\end{figure}
Pruned rollouts are immediately removed from the request pool, so the scheduler can reallocate the freed capacity to other active sequences, reducing the overall generation time without lowering GPU utilization.
\textbf{(ii) Frontend.}
The backend returns the pruning masks together with the generated rollouts.
The frontend filters out pruned rollouts and re-batches the surviving ones to compute log-probabilities and advantages, followed by policy optimization.
This reduces the log-probability computation and backpropagation cost roughly in proportion to the number of surviving rollouts.
\textbf{(iii) Quality head.}
Each rollout is naturally labeled by its final reward, so we can collect training data for the quality head on the fly.
We update the quality head with a cross-entropy loss, while stopping gradients to the backbone model.
As a result, the additional overhead is negligible. For more details, please refer to Fig.~\ref{fig:system} and Appendix~\ref{appx:algorithm}.

\subsection{Test-time Scaling.}
At test time, the trained quality head can also naturally serve as a correctness predictor.
Given a set of completed reasoning traces, we apply the head to obtain a score $s_i$ for each candidate.
Instead of a naive majority vote, we use these scores to calculate voting weights.
Since $s_i$ is an uncalibrated logit-like score, we convert scores to rank-based weights by sorting candidates and linearly rescaling their ranks to $[0,1]$. 
\section{Experiment}
\subsection{Experimental Settings} We build system based on verl~\citep{sheng2024verl} framework for training and vLLM~\citep{kwon2023efficient} framework as inference engine.
We conduct experiments on LLMs of different sizes and different series: Qwen-3 (1.7B, 4B, 8B), and LLaMA-3.2 (1.7B). We compare vanilla GRPO~\citep{grpo} and DAPO~\citep{dapo} with their \arrol-equipped variants on Math500~\citep{math}, OlympiadBench~\citep{olympiadbench}, and MinervaMath~\citep{minervamath} using average accuracy, and on AMC'23, AIME'24, and AIME'25 using pass@16. When training with our method, we use a cold-start period of 20 steps to initialize the quality head and the pruning estimator. For test-time scaling, we report maj@32 and compare our method against vanilla GRPO and trace score from DeepConf~\citep{fu2025deep}. More details about the metrics and datasets can be found in Appendix~\ref{appx:datasets}. 
Models are trained on the Dapo-Math-17K~\citep{dapo} dataset with a maximum sequence length of 8,192. The learning rate is set to $1\times 10^{-6}$, and the group size is set to 16. Other hyperparameters include $\kappa=0.5$, $\rho=0.5$, $L_\text{detect}=512$. For GRPO algorithm, we set $\epsilon_{low}=\epsilon_{high}=0.2$, and for DAPO algorithm, we change $\epsilon_{high}$ to 0.28. The experiments are conducted on NVIDIA GH200 GPUs. 
\subsection{Main Results}

\paragraph{Performances on GRPO.}

We report performance comparisons between vanilla GRPO and GRPO equipped with \arrol, as shown in Table~\ref{tab:grpo_comparison}. Across most benchmarks, \arrol\ consistently outperforms vanilla GRPO, improving average accuracy by \(+2.30\) to \(+2.87\) on the Qwen-3 series and by \(+2.86\) on LLaMA-3.2. Notably, the gains are larger on harder benchmarks: for example, \arrol\ improves AMC'23 by \(+7.50\) on Qwen-3-1.7B and improves AIME'24 by \(+10.00\) on Qwen-3-8B (and \(+6.67\) on AIME'25). Meanwhile, we observe small regressions on a few datasets (e.g., Minervamath on Qwen-3-1.7B/4B), but the overall improvement remains consistent.
In addition to better performance, \arrol\ substantially reduces training cost, achieving a stable \(1.6\!-\!1.7\times\) end-to-end speedup across model sizes. These results support a ``less is more'' effect: pruning yields fewer but more balanced samples, leading to both higher accuracy and better efficiency, and the gains are robust across model families and sizes. Finally, the quality head reaches \(\sim\)80\% prediction accuracy (e.g., \(82.37\%\) on Qwen3-1.7B), indicating it can be reliably trained within our pipeline for early-stage pruning decisions.

\paragraph{Performances on DAPO.} We further evaluate \arrol\ on DAPO by comparing vanilla DAPO with its \arrol-equipped variant. The results are reported in Table~\ref{tab:dapo_comparison}. Overall, \arrol\ achieves the best performance on most benchmarks, improving average accuracy by \(+2.99\) while maintaining a \(1.70\times\) end-to-end speedup. These results suggest that \arrol\ generalizes well across RLVR algorithms and delivers consistent efficiency gains.

\begin{table*}[!tb]

\centering{
\caption{Performance Comparison on GRPO. We compare our method with vanilla GRPO on six benchmarks across four models, and also report speedup.}{

\vspace{-7pt}
\setlength\tabcolsep{3pt}
    
\fontsize{10.5pt}{9.5pt}
\label{tab:grpo_comparison}
\begin{tabular}{l|cccccc|cc}
\toprule
  \textbf{Method}   &  \textbf{Math500} & \textbf{Minervamath} & \textbf{OlympiadBench} & \textbf{AMC'23} & \textbf{AIME'24} & \textbf{AIME'25} & \textbf{Avg} & \textbf{Speedup}\\
\midrule
     \multicolumn{9}{c}{\textit{\textbf{Qwen-3-1.7B-Base}}}\\
\midrule
GRPO & 60.89 & \textbf{17.65} & 18.55 & 75.00 & 20.00 & \textbf{16.67} & 34.79&-\\
\rowcolor{purple!10}+\arrol & \textbf{62.30} & 16.91 & \textbf{20.81} & \textbf{82.50}& \textbf{23.33} &\textbf{16.67} & \textbf{37.09}&1.61$\times$\\

\midrule
     \multicolumn{9}{c}{\textit{\textbf{Qwen-3-4B-Base}}}\\
\midrule
GRPO &79.64 & \textbf{30.88}  & 31.37& 87.50 & 36.67& 40.00 & 51.01 & -\\
\rowcolor{purple!10}+\arrol & \textbf{80.04} & 28.67 &\textbf{33.33} & \textbf{92.50} &\textbf{40.00} & \textbf{46.67} & \textbf{53.54}&1.63$\times$\\

\midrule
     \multicolumn{9}{c}{\textit{\textbf{Qwen-3-8B-Base}}}\\
\midrule
GRPO & 81.25 & 32.36 & \textbf{34.69} & \textbf{95.00}& 56.67 & \textbf{40.00} & 56.66&-\\
\rowcolor{purple!10}+\arrol & \textbf{81.45} & \textbf{34.19} & 33.18 & \textbf{95.00} & \textbf{66.67} & \textbf{46.67} & \textbf{59.53}&1.62$\times$ \\
\midrule
 \multicolumn{9}{c}{\textit{\textbf{LLama-3.2-1B-Instruct}}}\\
\midrule
GRPO  &24.20 &3.31& 2.26 & 45.00&13.33 & 0.00 & 14.63&-\\
\rowcolor{purple!10}+\arrol &\textbf{29.03} &\textbf{4.04} &\textbf{4.37} & \textbf{47.50}&\textbf{16.67} & \textbf{3.33} & \textbf{17.49} &1.67$\times$\\
\bottomrule
\end{tabular}
}
}ARRoL

\vspace{2pt}

\centering{
\caption{Performance Comparison on DAPO. We compare our method with vanilla DAPO on six benchmarks on Qwen-3-1.7B-Base, and also report speedup.}
\label{tab:dapo_comparison}

{
\setlength\tabcolsep{3pt}
\fontsize{10.5pt}{9.5pt}
\begin{tabular}{l|cccccc|cc}
\toprule
  \textbf{Method}   &  \textbf{Math500} & \textbf{Minervamath} & \textbf{OlympiadBench} & \textbf{AMC'23} & \textbf{AIME'24} & \textbf{AIME'25} & \textbf{Avg} &\textbf{Speedup}\\
\midrule
     \multicolumn{9}{c}{\textit{\textbf{Qwen-3-1.7B-Base}}}\\
\midrule
DAPO & \textbf{62.10} & \textbf{20.96} & 20.51 & \textbf{75.00}& 20.00 & 20.00 & 36.43&- \\
\rowcolor{purple!10}+\arrol & \textbf{62.10} & \textbf{20.96} & \textbf{20.97} & 72.50 & \textbf{33.33} &\textbf{26.67} & \textbf{39.42}&$1.70\times$ \\
\bottomrule
\end{tabular}
}
}
\end{table*}


\paragraph{Test-time Scaling.} To evaluate the quality head at inference time, we compare \arrol\ against vanilla majority voting and DeepConf~\citep{fu2025deep}, which uses log-likelihood-based trace confidence as voting weights for final-answer aggregation. As shown in Table~\ref{tab:test_time}, while DeepConf improves over majority vote, the learned quality head provides consistent gains across datasets and models, yielding up to \(+8.33\) additional improvement over DeepConf. These results suggest that the quality head trained during RLVR can serve as reliable confidence weights for test-time voting, outperforming model-intrinsic heuristic signals (e.g., DeepConf) that are not guaranteed to align with final-answer correctness.

\begin{table}
\caption{Performance Comparison of Test-time Voting. We compare our method with GRPO and Deepconf method on three benchmarks across three models.}
\label{tab:test_time}
\begin{center}
\begin{tabular}{l cccc}
\toprule
  \textbf{Method}  & \textbf{AMC'23} & \textbf{AIME'24} & \textbf{AIME'25}\\
\midrule
     \multicolumn{4}{c}{\textit{\textbf{Qwen-3-1.7B-Base}}}\\
\midrule
Majority & 55.0 & 16.7&3.3\\
Deepconf & 57.5&16.7 & 6.7\\
\rowcolor{purple!10}\arrol & \textbf{60.0} & \textbf{23.3} & \textbf{13.3}\\

\midrule
     \multicolumn{4}{c}{\textit{\textbf{Qwen-3-4B-Base}}}\\
\midrule
Majority & 72.5& 26.7 & 20.0\\
Deepconf &72.5 & 33.3 & 23.3\\
\rowcolor{purple!10}\arrol & \textbf{82.5} & \textbf{36.7} & \textbf{26.7}\\
\midrule
     \multicolumn{4}{c}{\textit{\textbf{Qwen-3-8B-Base}}}\\
\midrule
Majority & 75.0& 23.3 &26.7
\\
Deepconf & 80.0 & 23.3 &23.3\\
\rowcolor{purple!10}\arrol & \textbf{85.0} & \textbf{33.3} &  \textbf{33.3}  \\
\midrule
     \multicolumn{4}{c}{\textit{\textbf{LLama-3.2-1B-instruct}}}\\
\midrule
Majority & 10.0& 0.0& \textbf{0.0}\\
Deepconf & 15.0 & 3.3& \textbf{0.0}\\
\rowcolor{purple!10}\arrol  & \textbf{17.5} & \textbf{10.0}& \textbf{0.0}\\
\bottomrule
\end{tabular}
\end{center}
\end{table}

\subsection{Ablation Studies}
\begin{table}[!ht]
\begin{center}
\vspace{-4pt}
\caption{Performance Comparison against Random Pruning. $\hat{\rho}$ is the fraction of positive (reward=1) rollouts during training. For binary rewards, $\hat{\rho}(1-\hat{\rho})$ is proportional to the within-group reward variance and is maximized at $\hat{\rho}=0.5$.}
\label{tab:random}

\setlength\tabcolsep{2pt}
\hspace{-1pt}
{
\fontsize{10pt}{10.5pt}
\selectfont
\begin{tabular}{l ccccc}
\toprule
  \textbf{Method}  & \textbf{AMC23} & \textbf{AIME24} & \textbf{AIME25} & \textbf{$\mathbb{E}[{\hat{\rho}}]$}& \textbf{$\mathbb{E}[{\hat{\rho}}(1-\hat\rho)]$}\\
\midrule
     \multicolumn{6}{c}{\textit{\textbf{Qwen-3-4B-Base}}}\\
\midrule
Random & 79.34 & 26.47 & 31.98 & 0.32 & 0.21\\
\arrol &80.04 & 28.67 &33.33 & 0.40 & 0.23\\

\midrule
     \multicolumn{6}{c}{\textit{\textbf{LLama-3.2-1B-instruct}}}\\
\midrule
Random & 22.18 & 2.94 & 3.02 & 0.14 & 0.11\\
\arrol &  29.03 & 4.04 & 4.37 & 0.23 & 0.14\\
\bottomrule
\end{tabular}
}
\vspace{-12pt}
\end{center}
\end{table}

\begin{table}[!ht]
\begin{center}
\caption{Efficiency decomposition for Qwen-3-1.7B-Base across training phases: rollout generation, log-probability computation, and model update.}
\label{tab:decomposition}

\setlength{\tabcolsep}{2pt}
\renewcommand{\arraystretch}{0.8}
\fontsize{10pt}{10.5pt}
\selectfont
\begin{tabular}{l ccc}
\toprule
\textbf{Time/s} & \textbf{Generation} & \textbf{Logprob} & \textbf{Update} \\
\midrule
GRPO & 106.82 &18.40 &63.05\\
\arrol & 72.96 (1.46$\times$) & 10.02 (1.84$\times$)& 30.26 (2.08$\times$)\\

\bottomrule
\end{tabular}
\end{center}
\vspace{-12pt}
\end{table}

\begin{table}
\begin{center}
\caption{
Average accuracy on Math500, MinervaMath, and OlympiadBench, and training speedup under different rollout keep ratios $\kappa$ for Qwen-3-1.7B-Base.}

\label{tab:abl_r}

\begin{tabular}{l cccc}
\toprule
  \textbf{$\kappa$} &   \textbf{0.25} & \textbf{0.5} & \textbf{0.75} & \textbf{1} \\
\midrule
Avg Acc  & 32.46 & 33.34& 32.68& 32.36\\
Speedup & 2.33$\times$ &1.61$\times$ & 1.17$\times$& 1.00$\times$\\

\bottomrule
\end{tabular}
\end{center}
\vspace{-12pt}
\end{table}

\paragraph{Comparison with Random Pruning.}
To validate the effectiveness of \arrol, we compare it with random pruning in Table~\ref{tab:random}. \arrol\ consistently outperforms random pruning across datasets. 
We further report the within-group positive ratio during training. Specifically, for each prompt group we compute $\hat{\rho}$ as the fraction of positive (reward=1) rollouts during training, and report $\mathbb{E}[\hat{\rho}]$ and $\mathbb{E}[\hat{\rho}(1-\hat{\rho})]$ (average across groups). 
For binary rewards, $\hat{\rho}(1-\hat{\rho})$ is proportional to the within-group reward variance and is maximized at $\hat{\rho}=0.5$~\citep{bae2025online}, indicating stronger non-degenerate learning signals for group-normalized updates. 
As shown in Table~\ref{tab:random}, \arrol\ drives groups closer to balanced outcomes 0.5 and increases $\mathbb{E}[\hat{\rho}(1-\hat{\rho})]$, which is consistent with stronger learning signals and better final performance.

\paragraph{Efficiency Decomposition.} We further analyze the source of the efficiency gains by decomposing training time into different phases, as shown in Table~\ref{tab:decomposition}. Overall, our method accelerates all phases. For log-probability computation and model updates, the time is reduced by about \(2\times\), since pruning discards roughly half of the rollouts. In contrast, the speedup for rollout generation is smaller (\(1.46\times\)), because we first generate all sequences up to a threshold length \(L_{\text{detect}}\) before pruning half of the rollouts.

\paragraph{Ablation Study on keep ratio $\kappa$.} In our main experiments, we set the keep ratio $\kappa$ to 0.5. To study how $\kappa$ affects both effectiveness and efficiency, we evaluate several values of $\kappa$. As shown in Table~\ref{tab:abl_r}, a smaller $\kappa$ yields larger speedup since fewer rollouts are kept during pruning. Performance generally improves as $\kappa$ decreases, suggesting that pruning can also help by selecting more balanced sample subset. However, when $\kappa=0.25$, too many rollouts are removed, leading to a slight performance drop. Overall, $\kappa=0.5$ provides a good trade-off between accuracy and efficiency.

\paragraph{Wall-clock convergence.} We evaluate wall-clock convergence by plotting training reward against 
\begin{wrapfigure}{r}{0.6\columnwidth}
  \centering
  \hspace{-3pt}
  \vspace{-4pt}
  \includegraphics[width=\linewidth]{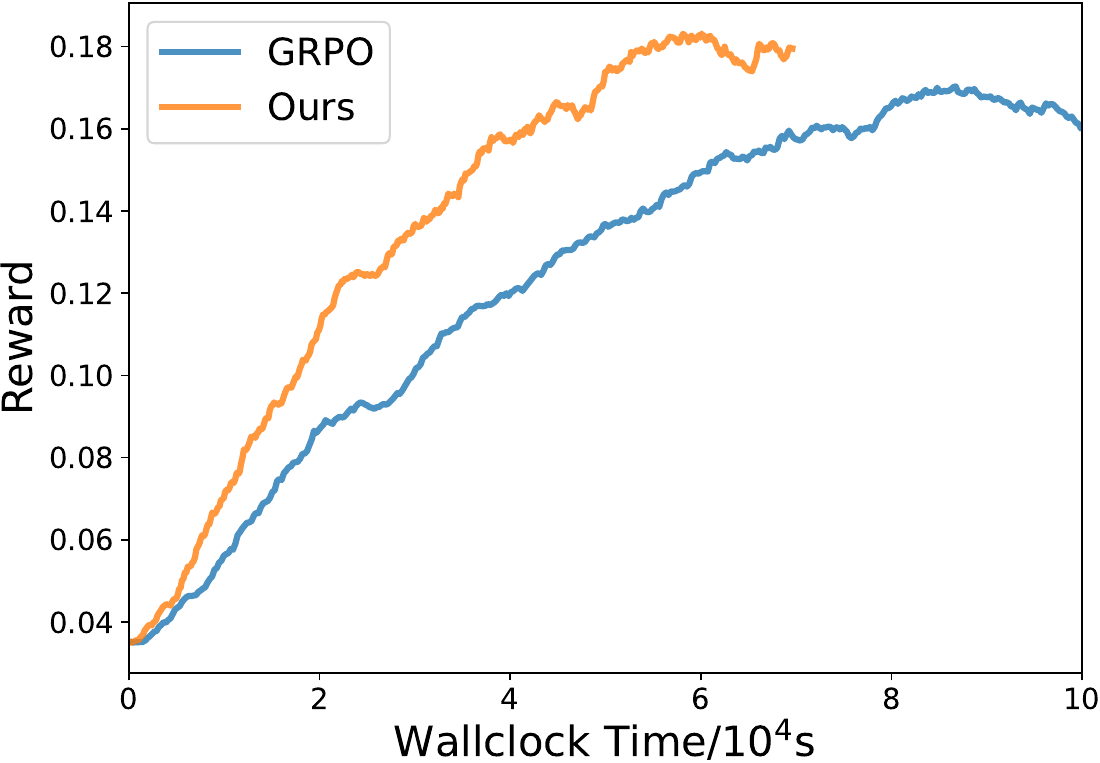}
  \caption{Wall-clock convergence of Qwen-3-1.7B-Base training.}
  \label{fig:convergence}
  \vspace{-10pt}
  \hspace{-3pt}
\end{wrapfigure}
wall-clock time, as shown in Fig.~\ref{fig:convergence}. Compared with GRPO, our pruning based training reaches the same reward level in less time and attains higher reward earlier across most of training, indicating improved time-to-reward and faster wall-clock convergence.

\section{Conclusion}
We presented \arrol, an online rollout pruning approach for RLVR that prunes rollouts \emph{during} generation while explicitly steering the surviving group toward a more balanced 0/1 reward composition, strengthening learning signals. \arrol\ trains a lightweight quality head on-the-fly to predict the success probability of partial rollouts, and uses them to make early pruning decisions under a target keep ratio. To realize efficiency gains, we further implemented a system design. Empirically, on different models, \arrol\ consistently improves accuracy while achieving up to 1.7$\times$ training speedup. Moreover, the learned quality head can also be used at test time as voting weights, yielding additional gains over naive majority voting. Overall, \arrol\ demonstrates a practical ``less rollouts, more learning'' paradigm for efficient and effective RLVR training and test-time scaling.

\section*{Limitations}
Our study mainly focuses on mathematical RLVR tasks with verifiable rewards; while the core idea (online pruning guided by a correctness predictor and balance control) is general and could be extended to other reward-based RL scenarios, such as UI interaction or tool-use agents, we do not validate these domains in this work. In addition, training-time pruning needs to generate tokens up to an intermediate detection length to evaluate partial rollouts (we set $L_{\text{detect}}=512$), so the rollout-generation speedup can be smaller than the savings in later phases because sequences must reach this threshold before pruning takes effect.

\bibliography{custom}

\newpage
\appendix

\section{Appendix}
\label{sec:appendix}

\subsection{Proof of Theorems in Sec.~\ref{sec:balanced}}

\label{appx:proof}
Consider a mini-batch of size \(G\), indexed by \(i\in\{1,\dots,G\}\).
Each sample has a binary label \(Y_i\in\{0,1\}\).
We assume conditional independence given latent Bernoulli parameters:
\[
Y_i \mid q_i^\star \ \sim\ \mathrm{Bernoulli}(q_i^\star),\qquad
\]
We assume $\{Y_i\}_{i=1}^G$ are independent given $\{q_i^\star\}_{i=1}^G$.
Here \(q_i^\star := \mathbb{P}(Y_i=1\mid X_i)\) is the true posterior.
We have an estimated posterior \(q_i\) satisfying a uniform accuracy condition
\[
|q_i - q_i^\star| \le \epsilon,\quad \forall i.
\]
For any index \(j\), define the \emph{true expected} positive ratio after removing \(j\):
\[
\mu^\star_{-j} \ :=\ \frac{1}{G-1}\sum_{i\ne j} q_i^\star.
\]
Likewise define the \emph{estimated} ratio based on \(\{q_i\}\):
\[
\mu_{-j} \ :=\ \frac{1}{G-1}\sum_{i\ne j} q_i.
\]
We consider the posterior-guided pruning rule
\[
\hat{\jmath}\ :=\ \arg\min_{j\in[G]} |\mu_{-j}-\rho|,
\]
where $\rho\in[0,1]$ is a fixed ratio.

\noindent
\textbf{Lemma~\ref{lemma:1}} (Existence of an improving prune)\textit{
If $\mu^\star>\rho$ and there exists an index $j$ such that $q_j^\star>\mu^\star$,
then pruning $j$ strictly reduces the deviation to $\rho$:
\[
\big|\mu^\star_{-j}-\rho\big| < \big|\mu^\star-\rho\big|.
\]
Symmetrically, if $\mu^\star<\rho$ and there exists $j$ such that $q_j^\star<\mu^\star$,
then the same conclusion holds.
}

\begin{proof}
Assume \(\mu^\star > \rho\). If \(q_j^\star > \mu^\star\), then
\[
\mu^\star_{-j}-\mu^\star
= \frac{G\mu^\star - q_j^\star}{G-1} - \mu^\star
= \frac{\mu^\star - q_j^\star}{G-1} < 0,
\]
so \(\mu^\star_{-j} < \mu^\star\). Since \(\rho < \mu^\star\), moving \(\mu^\star\) downward
moves it toward \(\rho\), hence
\(|\mu^\star_{-j}-\rho| < |\mu^\star-\rho|\).
The other case is symmetric.
\end{proof}
\begin{lemma}[Posterior error transfers to batch ratio]\label{lemma:A.1}
Given $|q_i - q_i^\star| \le \epsilon$, $\forall i$, for any \(j\), we have
\[
|\mu_{-j}-\mu^\star_{-j}|\ \le\ \epsilon.
\]
\end{lemma}
\begin{proof}
\begin{equation*}
\begin{split}
|\mu_{-j}-\mu^\star_{-j}|&=
\Big|\frac{1}{G-1}\sum_{i\ne j}(q_i-q_i^\star)\Big|\\
&\le \frac{1}{G-1}\sum_{i\ne j}|q_i-q_i^\star|\\
& = \frac{1}{G-1}(G-1)\epsilon
\le \epsilon.
\end{split}
\end{equation*}
\end{proof}

\begin{lemma}[Near-optimality of posterior-guided pruning]
Let \(j^\star:=\arg\min_j |\mu^\star_{-j}-\rho|\) be the best (oracle) removal index.
Then given $|q_i - q_i^\star| \le \epsilon$, $\forall i$, we have
\[
|\mu^\star_{-\hat{\jmath}}-\rho|
\ \le\ 
\min_{j}|\mu^\star_{-j}-\rho| + 2\epsilon
\ =\ 
|\mu^\star_{-j^\star}-\rho| + 2\epsilon.
\]\label{lemma:A.2}
\end{lemma}
\begin{proof}
By triangle inequality and Lemma~\ref{lemma:A.1},
\begin{equation*}
\begin{split}
|\mu^\star_{-\hat{\jmath}}-\rho|
&\le |\mu_{-\hat{\jmath}}-\rho| + |\mu_{-\hat{\jmath}}-\mu^\star_{-\hat{\jmath}}| \\
&\le |\mu_{-\hat{\jmath}}-\rho| + \epsilon.
\end{split}
\end{equation*}
Since \(\hat{\jmath}=\arg\min\limits_j|\mu_{-j}-\rho|\),

\begin{equation*}
\begin{split}
|\mu_{-\hat{\jmath}}-\rho|&\le |\mu_{-j^\star}-\rho| \\
&\le |\mu^\star_{-j^\star}-\rho| + |\mu_{-j^\star}-\mu^\star_{-j^\star}| \\
&\le |\mu^\star_{-j^\star}-\rho| + \epsilon.
\end{split}
\end{equation*}
Then we have \(
|\mu^\star_{-\hat{\jmath}}-\rho|
\ \le\ 
|\mu^\star_{-j^\star}-\rho| + 2\epsilon.
\)
\end{proof}

\begin{lemma}[Concentration of realized ratio around its expectation]
Let
\(
\hat{p}_{-\hat{\jmath}}:=\frac{1}{G-1}\sum_{i\ne \hat{\jmath}} Y_i\), \(
\mu^\star_{-\hat{\jmath}}:=\frac{1}{G-1}\sum_{i\ne \hat{\jmath}} q_i^\star,
\)
where \(Y_i\mid q_i^\star\sim \mathrm{Bernoulli}(q_i^\star)\) are conditionally independent.
Then for any \(t>0\),
\[
P\Big(\big|\hat{p}_{-\hat{\jmath}}-\mu^\star_{-\hat{\jmath}}\big|\ge t\ \Big|\ \{q_i^\star\}\Big)
\ \le\ 2\exp\big(-2(G-1)t^2\big).
\]\label{lemma:A.3}
\end{lemma}
\begin{proof}
Condition on the latent parameters \(\{q_i^\star\}_{i=1}^G\).
Further condition on the (possibly data-dependent) pruned index \(\hat{\jmath}\).
Given \(\hat{\jmath}=j\), the kept labels \(\{Y_i\}_{i\ne j}\) remain independent Bernoulli random variables with means
\(\mathbb{E}[Y_i\mid \{q_k^\star\}, \hat{\jmath}=j]=q_i^\star\) for all \(i\ne j\).
Define centered variables \(Z_i := Y_i-q_i^\star\) for \(i\ne j\).
Then \(\{Z_i\}_{i\ne j}\) are independent, satisfy \(\mathbb{E}[Z_i\mid \{q_k^\star\}, \hat{\jmath}=j]=0\),
and are bounded. Since \(Y_i\in\{0,1\}\) and \(q_i^\star\in[0,1]\),
\[
Z_i \in [-q_i^\star,\, 1-q_i^\star]\subseteq [-1,1].
\]
Additionally, we have
\[
\hat{p}_{-j}-\mu^\star_{-j}
=\frac{1}{G-1}\sum_{i\ne j} (Y_i-q_i^\star)
=\frac{1}{G-1}\sum_{i\ne j} Z_i.
\]
By Hoeffding's inequality~\citep{hoeffding1963probability},
for any \(t>0\), we have
\begin{equation*}
\begin{split}
&P\Big(\hat{p}_{-j}-\mu^\star_{-j}\ge t\ \Big|\ \{q_k^\star\}, \hat{\jmath}=j\Big)
\\
\le &\exp\Big(-\frac{2 (G-1)^2 t^2}{\sum_{i\ne j}(1-(-1))^2}\Big) \\
= &\exp\big(-2(G-1)t^2\big).
\end{split}
\end{equation*}
Applying the same bound to \(-(\hat{p}_{-j}-\mu^\star_{-j})\) and taking a union bound yields

\begin{equation*}
\begin{split}
P\Big(\big|\hat{p}_{-j}-\mu^\star_{-j}\big|\ge t\ \Big|\ \{q_k^\star\}, \hat{\jmath}=j\Big)
\\
\le 2\exp\big(-2(G-1)t^2\big).
\end{split}
\end{equation*}
Finally, remove the conditioning on \(\hat{\jmath}\):
\begin{equation*}
\begin{split}
&{P}\Big(\big|\hat{p}_{-\hat{\jmath}}-\mu^\star_{-\hat{\jmath}}\big|\ge t\ \Big|\ \{q_k^\star\}\Big)\\
=&\sum_{j=1}^G {P}(\hat{\jmath}=j\mid \{q_k^\star\}) \cdot\\
&\quad \quad P\Big(\big|\hat{p}_{-j}-\mu^\star_{-j}\big|\ge t\ \Big|\ \{q_k^\star\}, \hat{\jmath}=j\Big)\\
\le &2\exp\big(-2(G-1)t^2\big).
\end{split}
\end{equation*}
\end{proof}
\noindent
\textbf{Theorem~\ref{theorem:1}} (High-probability closeness to target $\rho$.)\textit{
Fix \(\delta\in(0,1)\). Given $|q_i - q_i^\star| \le \epsilon$, $\forall i$, with probability at least \(1-\delta\), we have
\[
\big|\hat{p}_{-\hat{\jmath}}-\rho\big|
\ \le\
\min_{j}|\mu^\star_{-j}-\rho|
\ +\ 2\epsilon
\ +\ \sqrt{\frac{\log(2/\delta)}{2(G-1)}}.
\]}
\begin{proof}
By triangle inequality, we have
\[
|\hat{p}_{-\hat{\jmath}}-\rho|
\le |\hat{p}_{-\hat{\jmath}}-\mu^\star_{-\hat{\jmath}}| + |\mu^\star_{-\hat{\jmath}}-\rho|.
\]
Use Lemma~\ref{lemma:A.3} with \(t=\sqrt{\frac{\log(2/\delta)}{2(G-1)}}\) and Lemma~\ref{lemma:A.2}, we can obtain the theorem.
\end{proof}

\subsection{Details of the Datasets}
\label{appx:datasets}
\paragraph{Dapo-Math-17k.} DAPO-Math-17k~\citep{dapo} is a curated collection of mathematical problems paired with verifiable final answers. The problems are sourced from online math resources and manual annotations, and are transformed to require an integer final answer to facilitate easy parsing. The dataset is commonly used for RLVR-style training and evaluation on math reasoning tasks.

\paragraph{Math500.} 
Math500~\citep{math500} is a subset of the MATH~\citep{math} dataset, containing 500 high-school-level problems. It covers a wide range of topics, including algebra, geometry, and precalculus, and is commonly used for comprehensive evaluation of mathematical reasoning.

\paragraph{Minervamath.} 
Minervamath~\citep{minervamath} consists of 272 problems, sourced primarily from MIT OpenCourseWare courses. It is designed to evaluate the mathematical and quantitative reasoning capabilities of LLMs.

\paragraph{OlympiadBench.} 
OlympiadBench~\citep{olympiadbench} contains 8,476 Olympiad-level math and physics problems, including problems from the Chinese college entrance exam. Each problem is accompanied by expert annotations with step-by-step reasoning.

\paragraph{AMC'23.} AMC'23 is a dataset of 40 problems from the 2023 American Mathematics Competitions (AMC). The final answer is an integer ranging from 0 to 999.

\paragraph{AIME'24/AIME'25.} Each dataset contains 30 challenging problems from the American Invitational Mathematics Examination (AIME). These questions require deep knowledge and techniques, especially in combinatorics and geometry. The final answer is an integer ranging from 0 to 999.

For small datasets such as AMC/AIME, repeated sampling is often used to reduce evaluation variance. 
Results are typically reported as avg@$k$, i.e., the average accuracy over $k$ independent repeats, or pass@$k$, 
i.e., whether at least one of the $k$ samples is correct. Maj@$k$ is also commonly used, defined as the accuracy of the majority-vote answer among $k$ repeats.

\subsection{Additional Related Work}
\paragraph{Efficient Large Language Models.} With the advancement in machine learning~\citep{xu2024slog,zeng2025pave,zeng2025harnessing,zeng2026subspace,wei2025cofirec} and foundation models~\citep{zhou2025scale,zhou2025dogr,zhang2025ta,qiu2022dimes,wei2026diffkgw,wei2026agentic,bei2026mem}, large language models (LLMs) have demonstrated significant potential across various domains, including mathematics~\citep{li2025beyond}, coding~\citep{zou2025latent,zou2025transformer}, question answering~\citep{chen2026influence}, complex reasoning~\citep{wei2026agentic}, recommendation~\citep{yoo2024ensuring,yoo2025embracing,yoo2025generalizable}, and multi-modality~\cite{bao2025latte,zeng2026subspace}. However, as the number of parameters grows, efficiency becomes a primary bottleneck for practical applications and deployment~\citep{wan2023efficient,lin2026efficient,cheninfluence}. To address this, various methods have been proposed to accelerate both training and inference. Efficient training methods primarily include Parameter-Efficient Fine-Tuning techniques such as LoRA and its variants~\citep{hu2022lora,ding2023parameter,wu2025sd,zeng2025hierarchical,qiu2026remix}, low-bit quantization~\citep{dettmers2023qlora}, and memory-efficient strategies like gradient checkpointing and 3D parallelism. Efficient inference methods focus on reducing computational costs through techniques such as post-training quantization~\citep{frantar2022gptq,lin2024duquant,lin2025quantization,yang2025lrq,bartan2025fineamp,zhang2026quantvla}, structural and non-structural pruning~\citep{lin2024mope,xing2025efficientllm,ai2025resmoe}, and knowledge distillation~\citep{gou2021knowledge} from larger teacher models. Furthermore, specialized LLM inference frameworks have been developed to optimize deployment; for instance, vLLM~\citep{kwon2023efficient} introduces PagedAttention to manage KV cache memory, while SGLang~\citep{zheng2024sglang} further optimizes complex LLM programs via the RadixAttention mechanism.

\paragraph{Reasoning in LLMs.} Reasoning has evolved from early symbolic AI and graph-based paradigms~\citep{he2026powergrow,xu2024discrete,qiu2023reconstructing} such as Knowledge Graph Reasoning (KGR)~\citep{liu2024logic}, to modern transformer-based architectures~\citep{vaswani2017attention,zeng2025hierarchical,zeng2025interformer}. While early neural methods utilized GNNs~\citep{zeng2023parrot,zeng2023generative,zeng2024hierarchical,zeng2024graph,yu2025joint,yu2025planetalign,qiu2026graph,qiu2024tucket} to bridge structured data with vector representations, the LLM era has shifted focus toward emergent reasoning through Chain-of-Thought (CoT) prompting~\citep{dou2025enhancing}. Recent advancements further enhance these capabilities via post-training reinforcement learning, employing techniques with RLVR, like  GRPO~\citep{grpo}, Dr.GRPO~\citep{liu2025understanding}, and DAPO~\citep{dapo}. These methods move beyond simple next-token prediction by utilizing verifiable rewards to encourage multi-step exploration and self-correction, enabling models to solve complex logical and mathematical problems more reliably.

\subsection{Details of Calibration Mapping and Survival Probability Design}

\label{appx:estimator}
As stated in Sec.~\ref{sec:quality}, in practice, our quality head outputs an uncalibrated scalar \emph{quality score} $s_i$ for each rollout. We first normalize the score to $s_i\in[0,1]$ using sigmoid function. We therefore require a calibration mapping $q(s')=P(Y=1|s_i')$ to convert scores into posterior estimates. We design an online binned probability calibration.

We employ an estimator with $B$ bins. 
Specifically, we map a score $s'$ to a bin index $b(s')=\min(B-1,\lfloor Bs'\rfloor)$.
We maintain two labeled buffers of historical rollouts (positive/negative outcomes) and compute histogram counts
\begin{equation}
\begin{split}
\label{eq:count}
c_b^{+}=\#\{k:b(s_k^{+})=b\},\\
c_b^{-}=\#\{k:b(s_k^{-})=b\}.
\end{split}
\end{equation}
To avoid the situation where $c_b^{+/-}=0$ in some bins and reduce few-sample noise, 
we further utilize Laplace smoothing $\alpha$, and estimate the class-conditional likelihoods:
$$P(b|Y=1)=(c_b^+ + \alpha)/(\sum_{b'\in [B]} c_{b'}^+ + \alpha B),$$
$$P(b|Y=0)=(c_b^- + \alpha)/(\sum_{b'\in [B]} c_{b'}^- + \alpha B).$$

Given a prior $\pi=P(Y=1)$ (estimated from the buffers),
we compute posterior-mean estimates via Bayes' rule:
\begin{equation*}
\begin{split}
&q(s)=P(Y=1\mid b(s))\\&=\frac{\pi P(b(s)\mid Y=1)}{\pi P(b(s)\mid Y=1)+(1-\pi)P(b(s)\mid Y=0)}.
\end{split}
\end{equation*}

Based on $q_i=P(Y=1|s_i)$, we assign each rollout a survival probability using an affine function of the deviation from $\rho$:
\begin{equation}
p_i=\text{clip}(\kappa+\delta+\lambda(\rho - q_i),p_{\min},p_{\max}),
\label{eq:survival}
\end{equation}
where $\kappa$ is the target keep rate, $\lambda$ controls the strength of balance correction, and $\delta$ is a scalar bias. The design goal is: (i) the expected keep ratio matches a target $\kappa$,
and (ii) the kept rollouts have a controlled positive
ratio close to $\rho$. 
We solve for $\delta$ by binary search such that the expected keep rate matches $\kappa$ under the current buffer distribution:
$\mathbb{E}\,[p_i]=\kappa$.
Finally, clipping to $[p_{\min},p_{\max}]$ prevents overly aggressive pruning and ensures every rollout retains a non-zero chance to survive, preserving exploration diversity.

In practice, we use $\lambda=0.5$, $B=128$.

\subsection{Algorithm Framework}
\label{appx:algorithm}
Here we present an algorithm framework to better illustrate our system containing fronend and backend, as shown Algorithm~\ref{alg:framework}.

\begin{algorithm*}[!ht]
\caption{\arrol\ System}
\label{alg:framework}
\begin{algorithmic}[1]
\Require Dataset $\{(x_i,a_i)\}_i$, batch size $B$, group size $G$
\Statex \textbf{FRONTEND (verl)}
\For{training step $t=1,2,\dots$}
    \State Sample prompts $\{(x_i,a_i)\}_{i=1}^{B}$ and form $B\!\times\!G$ rollout requests $\mathcal{P}$
    \State Compute histogram params $h_t$ (binned estimator in Eq.~\eqref{eq:count}, Eq.~\eqref{eq:survival})
    \State Stream current policy + quality-head weights to backend
    \State Send $(\mathcal{P}, h_t)$ to backend and wait for responses $\mathcal{R}$ (with prune flags)
    \State Filter pruned rollouts $\mathcal{R}\to\mathcal{R}'$, rebatch
    \State Compute rewards, log-probs on $\mathcal{R}'$; main RL loss + quality-head loss
    \State Backprop and optimizer step on policy and quality head
\EndFor

\Statex \textbf{BACKEND (vLLM)}
\While{requests arrive}
    \State Receive $(\text{weights stream}, \mathcal{P}, h)$; load/attach weights
    \State Insert $\mathcal{P}$ into request pool
    \While{request pool not empty}
        \State Adaptively pick a micro-batch $\{r_i\}_{i=1}^{B'}$ from the pool
        \State Do one forward step (prefill/decoding) to get next tokens and quality logits
        \For{each $r_i$ in the micro-batch}
            \If{$r_i$ hits $L_{\text{detect}}$ for the first time}
                \State Compute count $s_i$ and survival prob $p_i$ (Eq.~\eqref{eq:count}, Eq.~\eqref{eq:survival})
                \State Prune $r_i$ w.p. $1-p_i$; record prune flag
            \EndIf
            \State Remove completed/pruned requests from pool; mark prune flags
        \EndFor
    \EndWhile
    \State Return all responses (with prune flags) to frontend
\EndWhile
\end{algorithmic}
\end{algorithm*}

\subsection{Potential Risks}
We only use public data, and we do not expect any personally identifying information.
Our method reduces the cost of RL training, which could lower the barrier to scaling RL-based post-training and be misapplied outside math. We recommend standard safety policies and dataset hygiene when transferring to other domains.

\subsection{Declaration of AI Assistance}
We used AI tools solely for language polishing. These tools did not contribute to the experiments, analysis, results, or scientific claims, and no paragraphs were generated by AI.

\end{document}